\documentclass{article}
\usepackage{natbib}
   
\usepackage{amsmath}
\usepackage{mathtools}
\usepackage{amsthm}
\usepackage{stmaryrd}
\usepackage{dsfont}

\usepackage{url}

\usepackage{marginnote}

\usepackage{enumitem}

\usepackage{cancel}

\mathtoolsset{showonlyrefs}

\newcommand{\AIXI}{\textsc{AIXI}}
\newcommand{\UCAI}{\textsc{UCAI}}
\newcommand{\UAI}{\textsc{UAI}}
\newcommand{\AIXItl}{\textsc{AIXI}$tl$}
\newcommand{\MCAIXI}{\textsc{MC-AIXI(FAC-CTW)}}

\newcommand{\dwcable}{digit-wise computable}

\newcommand{\dbl}[1]{\left\llbracket#1\right\rrbracket}

\newcommand{\xs}{x_\infty}

\newcommand{\ys}{y_\infty}
\newcommand{\zs}{z_\infty}
\newcommand{\sss}{s_\infty} %
\newcommand{\opluss}{\oplus_\infty}
\newcommand{\masae}{\mbox{\ae}}
\DeclareMathOperator*{\append}{+\!\!\!+} %

\newcommand{\soesum}[1]{\sum\limits_{\substack{
                                 q\in#1,\\
                                 \dbl{q}(a_{1..m(k)}) = e_{1..m(k)}
                               }}}
\newcommand{\soes}[2]{\sum\limits_{\substack{
                                 q\in#2,\\
                                 \dbl{q}(a_{1..#1}) = e_{1..#1}
                               }}}

\newcommand{\res}{\left(\sum_{t=k}^{m(k)}r(e_t)\right)}

\newcommand{\aref}[1]{Appendix~\ref{#1}}

\newcommand{\terminating}{model of terminating computation}
\newcommand{\terminatings}{models of terminating computation}

\newcommand{\lhs}{left-hand side}
\newcommand{\rhs}{right-hand side}

\newcommand{\cmp}{\mathbin{?}}

\DeclareMathOperator*{\argmax}{arg~max} %

 \newtheorem{definition}{Definition}
\newtheorem{lemma}{Lemma}
\newtheorem{theorem}{Theorem}
\newtheorem{corollary}{Corollary}
\newtheorem{remark}{Remark}

\begin{document}

 \title{Computable Variants of \AIXI{} \\ which are More Powerful than \AIXItl{}}
  \author{Susumu Katayama, University of Miyazaki}
\date{Jan. 24, 2019}

 \maketitle

 \begin{abstract}
   This paper presents Unlimited Computable AI, or \UCAI{}, that is a family of computable variants of \AIXI{}.
  \UCAI{} is more powerful than \AIXItl{}, which is a conventional family of computable variants of \AIXI{}, in the following ways:
 1) \UCAI{} supports \terminatings{}, including typed lambda calculi,
    while \AIXItl{} only supports Turing machine with timeout $\tilde{t}$, which can be simulated by typed lambda calculi for any $\tilde{t}$;
 2) unlike \UCAI{}, \AIXItl{} limits the program length to some $\tilde{l}$.
 \end{abstract}

 \section{Introduction}

 \AIXI{} \citep{topdown} is an AI model for theoretical discussion of the limitations of AI. \AIXI{} is expected to be universal.
\AIXI{} models the environment as a Turing machine and formalizes the interaction between the environment and the AI agent as a discrete-time {\em reinforcement learning} \citep[e.g.,][]{SuttonBarto} problem, that is an optimization problem under unknown environment.
\AIXI{} is not computable without approximation.

 \AIXItl{} and AIXI$\tilde{t}\tilde{l}$ \citep{topdown} are computable variants of \AIXI{}, but they have limitations on the description length of the environment and the computation time for each time step when computed by a sequential Turing machine.
 Especially, the limitation on the program description length may be a problem when dealing with the universe whose size is unknown beforehand.

In this paper, we introduce 
 AI models that generalize
  the definition of \AIXI{} to support a broader class of models of computation as the environment model than Turing machines.
Then, we show
  the computability of our AI models 
  adopting \terminatings{} such as typed $\lambda$ calculi.
Because the model of computation used by \AIXItl{} and AIXI$\tilde{t}\tilde{l}$ is ``Turing machines with timeout $\tilde{t}$ within length $\tilde{l}$'', which is a more limited class than structural recursive functions, our AI models %
 cover a more general environment class.
We claim that our introduction of such computable AI models which are more powerful than known models is useful for research on the theoretical limitations of computable AI models.

 We prove two main
 theorems about Unlimited AI (\UAI{}), that is our formalization of generalized \AIXI{}, when using a \terminating{} as the environment model.

 The first theorem shows that the {\em action value function} (representing how valuable each action in each situation is) of \UAI{} is computable to arbitrary precision.
 Roughly speaking, this result means that \UAI{} is at least as computable as $Q$-learning \citep{WD92}:
 \UAI{} and $Q$-learning both select an action at each time step by $\argmax$ operation on real-valued action value functions,
 but strictly speaking, the $\argmax$ operation is not exactly computable if different actions may have the same action value, because a comparison of exactly the same real values requires comparisons of an infinite number of digits.
 Practically, however, $Q$-learning is considered computable, by giving up computation at some precision.

 The other theorem shows that
 \UAI{} becomes exactly computable including the $\arg\max$ operation
   with probability 1 by randomly selecting the least significant bits of each prior probability.
   We call the resulting AI model Unlimited Computable AI (\UCAI{}).

Making \AIXI{} computable involves solving the following three problems:
\begin{description}
 \item[Problem 1] programs that are candidate models of the environment can enter infinite loops;
 \item[Problem 2] the agent needs to compute an infinite series because there are an infinite number of candidate models of the environment;
 \item[Problem 3] all operations are on real values. %
\end{description}

 \AIXItl{} deals with the above %
 problems in the following way:
\begin{enumerate}
 \item it deals with Problem~1 by limiting the computation time by introducing a timeout;
 \item it deals with Problem~2 by limiting the number of programs finite by limiting the length of programs.
\end{enumerate}
 Although Problem~3 must also be considered for strictly theoretical discussion, it is just ignored in the papers on \AIXItl{} \citep[e.g.,][]{topdown}, and it is not explicitly solved.
 However, if the set of rewards is limited to rational numbers, all the computation can be executed as operations on rational numbers, because the number of programs is
 finite
 in \AIXItl{}.
 Throughout this paper, we assume that the reward can only take rational numbers.

On the other hand, our \UCAI{} algorithm deals with the three problems in the following way:
\begin{enumerate}
 \item
 by not limiting the environment model to Turing machines and permitting \terminatings{}, it deals with Problem~1 while supporting models of computation which are more powerful than Turing machines with any timeout;
 \item
  it deals with Problem 2 and 3 without limiting the program length by applying
  the technology called {\em exact real arithmetic} \citep[e.g.,][]{BoehmCartwright1993}, which enables exact computation with real values under some limitations.
  Intuitively, the idea is to represent bounded real values as lazy infinite streams of digits.
\end{enumerate}

The rest of this paper is organized in the following way. 
Section~\ref{sec:AIXI} introduces \AIXI{}.
Section~\ref{sec:exactReal} introduces exact real arithmetic and its implementation using lazy evaluation.
 Section~\ref{sec:contributions} defines \UCAI{} %
and proves that \UCAI{} is exactly computable.
 Section~\ref{sec:conclusions} summarizes the results and discusses what the results mean.

\section{Preparation}

This section provides the definitions of the concepts used in this paper.

\subsection{Finite List}
Before defining \AIXI{}, we define (finite) lists, because \AIXI{} deals with the discrete time sequence. Infinite lists or streams will be defined in Section~\ref{sec:exactReal}.

\begin{definition}[Finite List]\label{def:list}
 A finite list is either the empty list $[]$ or the result of prepending an element to a finite list.
 $:$ denotes the prepending operator, corresponding to the cons function in Lisp.

 $:$ is right associative.
 For example, the list $x_1:(x_2:(x_3:[]))$ can be written as $x_1:x_2:x_3:[]$ by omitting the parentheses,
 which is the list consisting of $x_1$, $x_2$, and $x_3$ in this order.

 $A{*}$ denotes the type of lists of $A$'s for any type $A$ in general.
\end{definition}

For lists we define the concatenation operator $\append$, such that $(a_1:a_2:[])\append (b_1:b_2:b_3:[]) = a_1:a_2:b_1:b_2:b_3:[]$.
\begin{definition}[Concatenation]
 \begin{align}
     []      \append y &= y \\
     (x_1:x) \append y &= x_1 : (x \append y)
 \end{align}
\end{definition}
$\append$ is also right-associative.

\subsection{AIXI}\label{sec:AIXI}

 \AIXI{} \citep{topdown} is an AI model for theoretical discussion of the limitations of AI. \AIXI{} is expected to be universal.

 Mostly, \AIXI{} is based on %
 the general reinforcement learning framework:
\begin{itemize}
 \item at each time step, the learning agent and the environment interact via {\em perceptions} and {\em actions}: the agent chooses an action based on the interaction history and sends it to the environment, then, the environment chooses a perception based on the history and sends it to the agent, then again the agent chooses an action based on the new history and sends it to the environment, and so on;
 \item the perception includes the information of {\em reward}, which reflects how well the agent has been behaving; the agent's purpose is to maximize the {\em expected return}, that is the expectation of the total sum of future reward;
 \item the environment is unknown to the agent.
     It may even change in time.
\end{itemize}

\AIXI{} models the environment as a Turing machine. It estimates the environment in the way that higher prior probabilities are assigned to simpler programs.
 At each time step, it selects the action which maximizes the expected return weighted by the belief assigned to each Turing machine.

  \begin{remark}
  Turing machines and other models of computation also have the idea of the number of computation steps for measuring the computation time.
  Although we do not mention each computation step in this paper, we often use the term ``interaction step'' instead of ``time step'' in order to emphasize that we are not talking about computation steps.
  \end{remark}

\AIXI{} has the following parameters:
\begin{itemize}
 \item the finite non-empty set of actions ${\cal A}$,
 \item the finite non-empty set of observations ${\cal O}$,
 \item the finite and bounded set of rewards ${\cal R} \subset [r_{\min}, r_{\max}]$, and %
 \item the horizon function $m \in \mathds{N}\rightarrow\mathds{N}$ such that $m(k)\ge k$.
\end{itemize}
Although \citet{topdown} assumes non-negative rewards, we omit this limitation because this is requested only by \AIXItl{} which is an \AIXI{} approximation, and
also because this is easily amendable.
Also, we do not often mention ${\cal O}$, but instead, the set of perceptions ${\cal E} = {\cal O} \times {\cal R}$ and the projection function $r \in {\cal E} \rightarrow {\cal R}$.
\footnote{Throughout this paper, we use `$\in$' instead of `$:$' even for functions for readability, in order to avoid the name collision against the `$:$' operator defined in Definition~\ref{def:list}. Its fixity is weaker than $\rightarrow$, and thus $r \in {\cal E} \rightarrow {\cal R}$ means $r \in ({\cal E} \rightarrow {\cal R})$.}

At each time step $k$, \AIXI{} computes the action $a_k$ from the interaction history based on the following equations.
 Firstly, the action value function $Q \in ({\cal A}\times{\cal E}){*} \times {\cal A} \rightarrow \mathds{R}$ computing the expected return for each action based on the current history
 is defined as follows:
\begin{align}
 \MoveEqLeft Q(\masae_{1..k-1}, a_k)\\
 &= \sum_{e_k\in {\cal E}} \max_{a_{k+1}\in{\cal A}} \sum_{e_{k+1}\in {\cal E}} \max_{a_{k+2}\in{\cal A}} ... \sum_{e_{m(k)-1}\in {\cal E}} \max_{a_{m(k)}\in{\cal A}} \sum_{e_{m(k)}\in {\cal E}} \\
 &\quad \left(\res\soesum{T}\xi(q)\right) %
\end{align}
 where $e_k$ denotes the observation at time $k$.
 In general, $v_{m..n}$ denotes $v_m : v_{m+1} : ... : v_{n-1} : v_n : []$, %
 i.e., the sequence from time $m$ to time $n$ of time-varying variable $v$.
 $\masae_k$ denotes $(a_k,e_k)$ that is the action-perception pair at time $k$.
 Thus, $\masae_{1..k-1}$ is the interaction history $(a_1,e_1):...:(a_{k-1},e_{k-1}):[]$ at time $k$.
Also note that
\begin{align}
a_{1..m(k)} &= a_{1..k-1}\append a_k:a_{k+1..m(k)} \\
e_{1..m(k)} &= e_{1..k-1}\append e_{k..m(k)}
\end{align}
(because $1\le k\le m(k)$), and that $a_{1..k-1}$, $a_k$, $a_{k+1..m(k)}$, $e_{1..k-1}$, and $e_{k..m(k)}$ are bound in the different ways.

$T$ is the set of monotone Turing machines.
$\xi \in T \rightarrow [0,1]$ is the {\em universal prior} defined as
\begin{align}
 \xi(q) &= 2^{-l(q)} \label{eq:originalUniversalPrior}
\end{align}
where $l(q)$ is the length of $q$ in the prefix code.
The universal prior is designed to prefer simple programs by assigning lower probabilities to longer programs.

\begin{equation}
 \dbl{q}(a_{1..m(k)}) = e_{1..m(k)} \label{eq:dblqAIXI}
\end{equation}
 denotes the condition that $q$ behaves interactively,
 taking $a_1$ as the input, returning $e_1$ as the output,
 taking $a_2$ as the input, returning $e_2$ as the output, and so on, and
 taking $a_{m(k)}$ as the input, and returning $e_{m(k)}$ as the output.

Based on the action value function $Q$ defined above, \AIXI{} always chooses the best action, i.e., the action $\dot{a}_k$ at time $k$ after the actual interaction history $\dot{\masae}_{k-1}$ such that
\begin{equation}
 \dot{a}_k = \argmax_{a\in{\cal A}} Q(\dot{\masae}_{1..k-1}, a) \label{eq:AIXIa}
\end{equation}

\subsection{Exact Real Arithmetic and Lazy Evaluation}\label{sec:exactReal}
 Exact real arithmetic \citep[e.g.,][]{BoehmCartwright1993} is a set of techniques for effectively implementing exact computations over the real numbers.
 Because the set of real numbers is a continuum and each real number contains the information of an infinite number of digits%
, the reader may doubt if it is even possible.
 Actually, that is not a problem because the set of real numbers that can be uniquely defined by a finite program is countable.

 There are two main kinds of approaches to representing exact real values.\citep{Plume98}
 One represents them as a function taking the precision returning an approximation to the precision \citep{BoehmCartwright1993}, and the other represents the mantissa as a lazy infinite stream of digits or other integral values.
The latter stream-based approach has various representations, which are reviewed by \citet{Plume98}.

 Lazy functional languages such as Haskell adopt the lazy evaluation model
 and can deal with infinite data structures such as infinite lists (a.k.a. streams) and infinite trees.
 The idea of lazy evaluation is to postpone computation until it is requested.
 Even if the remaining computation will generate infinite data, the data structure can hold a {\em thunk}, or a description of the remaining computation, until a more precise description is requested.

 Now we give a definition of a stream.
\begin{definition}[Stream]\label{def:stream}
 A {\em stream} is an infinite list.
 An infinite list is the result of appending an element in front of an infinite list using the cons $:$ operator.
 (We use the same letter $:$ for both finite and infinite lists.)
 
 Again, $:$ is right associative.
 For example, the stream $x_1:(x_2:(x_3:...$ can be written as $x_1:x_2:x_3:...$ by omitting the parentheses.

 $A\infty$ denotes the type of streams of $A$'s for any type $A$ in general.
\end{definition}
 This is an intuitive definition.
 More interested readers should read a functional programming textbook \citep[e.g.,][]{BirdWadler}.

In this paper, stream variables are suffixed with $\infty$ for readability, such as $x_\infty$.

We prove the computability of our \AIXI{} variants when using redundant binary (a.k.a. signed binary) stream representation.
For simplicity, we normalize the set of rewards in order for all the computations to be executed within $(0,1)$, and stick to fixed point computations.
\begin{definition}[Fixed point redundant binary stream representation, FPRBSR]
 Let $x \in [-1,1]$.
 A {\em fixed point redundant binary stream representation}, or {\em FPRBSR} in short, of $x$ is a stream $x_1:x_2:x_3:...$ where
 \[
  x = \sum_{i=1}^\infty 2^{-i}x_i
 \]
 and
 \[
  x_i \in \{-1, 0, 1\}
 \]
 hold for each $i$.

 $\dbl{\cdot} \in \{-1, 0, 1\}\infty \rightarrow [-1,1]$ denotes the interpretation of the given FPRBSR, or conversion to $[-1,1]$.
 \[
  \dbl{x_1:x_2:x_3:...} = \sum_{i=1}^\infty 2^{-i}x_i
 \]
 holds.

\end{definition}

 In order to make sure that such infinite computations generate results infinitely, namely, to make sure the computability to arbitrary precisions,
 it is enough to show that some initial element(s) (or digit(s), in the case of a stream of digits) and the data representing the remaining computation can be computed at each recursion.

\begin{definition}[Digit-wise computability]
 Function $f$ returning an FPRBSR is {\em digit-wise computable} iff it infinitely generates the resulting stream, i.e., for any precision $p$ there always exists the time $t$ when the result is computed to the precision $p$.
\end{definition}

Digit-wise computability is the computability of FPRBSR.
It means the value can be computed to any given precision.
it is just called computability in the literature \citep[e.g.,][]{Plume98}.

Digit-wise computability should not be confused with computability in the limit.
Obviously, the former is stronger than the latter.
The results which are obtained during the process of digit-wise computations are exact and definite, and can never be overwritten. %
On the other hand, computability in the limit does not care about the process nor guarantee any truth of the results obtained during the computation.

In order to prove that the function $f$ is \dwcable{}, it is enough to show that $f$ can be represented as
\[
 f(x) = g(x) : (h(f))(x)
\]

 We need to request that once the first elements of streams, or the most significant digits in the case of streams of digits, are computed and fixed, they must not change later by carrying. %
 In fact, the usual binary representation does not always satisfy this rule: e.g., when we are adding $0.01010...$ to $0.00101...$, whether
 the most significant digit %
 of the result
 is $0$ or $1$
 may be undecided
 forever.
 Instead, we use the redundant binary representation which only requests finite times of carrying out for most arithmetic operations, for avoiding this problem.

The most serious limitation of exact real arithmetic is that comparisons between two exactly the same numbers do not terminate (unless we know that they are the same beforehand), because such comparisons mean comparisons between exactly the same infinite lists.
 We can still compare two different numbers to tell which is the greater.

 In the defining equations of \AIXI{}, this limitation only affects their $\arg\max$ operations. %
 (It does not affect the $\max$ operations, as shown in Lemma~\ref{lemma:max}.)
 This means that it is very difficult to exactly choose the best action %
  when the best and the second best action values are almost the same.
 In practice, however, there may be cases where we can ignore such small differences and use fixed-precision floating point approximations instead of exact real numbers.

 This paper proves that the action values of our \AIXI{} generalization can be computed to arbitrary precisions for arbitrary prior distributions when a \terminating{} instead of a Turing machine is used (Theorem~\ref{th:UAIQ}), and that the whole computation, including the $\arg\max$ operation, is exactly computable
  with probability 1 if we modify the least significant digits of probabilities of the prior distribution to random irrational values
  (Theorem~\ref{th:UCAI}).

\section{Contributions}\label{sec:contributions}

In this section, we introduce our \AIXI{} variant which supports more powerful models of computation than what \AIXItl{} uses and is still computable.
For this, we start with generalizing \AIXI{} to the necessary level, and then we specialize it to obtain computable models.

\subsection{Generalizing \AIXI{}}

In this section, we generalize \AIXI{} by generalizing the model of computation from monotone Turing machine to other models and generalizing the prior distribution function on it.

\subsubsection{Generalizing the Model of Computation}\label{sec:generalize}

  \AIXI{} uses universal prefix Turing machine, that executes any program in a prefix code.
 The program is usually written as an encoding of a Turing machine or in a Turing-complete language,
 but it can be written in a terminating language such as typed $\lambda$ calculi and functional languages extending them.

 Although the reader may think that using a Turing-incomplete model of computation is unacceptable,
 it should be better than \AIXItl{},
 because ``Turing machines with a timeout'', which \AIXItl{} uses, can be simulated by typed $\lambda$ calculi.
 
  When using a terminating functional language, the effect of using {\em monotone} Turing machines can be achieved by using lazy I/O, %
 i.e., by modeling the input and the output as streams.
 (Moreover, even when streams are not available, the same computational ability is achieved by supplying the interaction history instead of supplying only the current perception as the current input, though the efficiency is sacrificed.)

 The idea of termination and lazy I/O can coexist. For example, Agda is a computer language equipped with both.
 The careful reader should recall that similarly monotone Turing machines may or may not enter an infinite loop while computing
 each output in reply to each input.

Extending the available set of models of computation to include $\lambda$ calculi and functional languages has another bonus of enabling incremental learning by assigning biased prior probabilities to prioritize expressions with useful functionality.\citep{Katayama16}
We do not discuss this further in this paper.

\subsubsection{Assigning Prior Distributions}\label{sec:assign}
  For any model of computation ${\cal M}$ not limited to Turing machines,
 its universal prior can be defined in the same way as Eq.~\eqref{eq:originalUniversalPrior}:
 \begin{align}
 \xi_1 &\in {\cal M} \rightarrow [0,1] \\
 \xi_1(q) &= 2^{-l(q)} \label{eq:generalPrior}
\end{align}
  where $l(q)$ denotes the length of $q$ in the prefix code. %
 The prefix code should be a compressed one based on the grammar
 when considering efficiency,
 though usual programs in plain text are already in prefix codes.
 Adequate selection of a concise prefix code helps to avoid grammatically-incorrect programs.
 
Note that there can be grammatically-correct, but type-incorrect programs, in the same way as non-terminating programs of Turing machines.
Such programs should be selected with probability $0$, and the actual probabilities for valid programs should be
those divided by $\sum_{q\in{\cal V}}\xi_1(q)$, where ${\cal V}$ denotes the set of valid programs.
 Operationally, this is equivalent to retrying generation of syntax tree when
  Eq.~\eqref{eq:generalPrior}
  results in an invalid program.
There is no problem, because the probability of generating a valid terminating program based on
  Eq.~\eqref{eq:generalPrior}
  is constant if the language is fixed,
 and because the result of Eq.~\eqref{eq:AIXIa} would not be affected by multiplication by a constant.

For any finite program $q \in {\cal M}$, $\xi_1(q)$ is rational, though the actual prior
\[
 \frac{\xi_1(q)}{\sum_{q\in{\cal V}}\xi_1(q)}
\]
may not.

\subsubsection{Formalizing Generalized \AIXI{}}

\UAI{}, or {\em Unlimited AI}, is our formalization of generalized \AIXI{} supporting the above generalizations.
A \UAI{} has the following parameters:
\begin{itemize}
 \item the non-empty finite set of actions ${\cal A}$,
 \item the non-empty finite set of observations ${\cal O}$,
 \item  the finite and bounded set of rational rewards ${\cal R} \subset [r_{\min}, r_{\max}] \wedge \mathds{Q}$ where $r_{\min}<r_{\max}$ and $r_{\min}, r_{\max}\in \mathds{Q}$, %
 \item the horizon function $m \in \mathds{N}\rightarrow\mathds{N}$ such that $m(k)\ge k$,
 \item the model of computation ${\cal M}$, which can be viewed as a set of programs taking a lazy list of actions and returning a lazy list of perceptions,
 \item the prior distribution function $\rho \in {\cal M} \rightarrow [0,1]$
\end{itemize}
\UAI{}$({\cal A},{\cal O},{\cal R},m,{\cal M},\rho)$ denotes \UAI{} with the above parameters.
The definition of the set of perceptions ${\cal E} = {\cal O} \times {\cal R}$ and the projection function $r \in {\cal E} \rightarrow {\cal R}$ are the same as those of \AIXI{}'s.

The action value function $Q \in ({\cal A}\times{\cal E}){*} \times {\cal A} \rightarrow \mathds{R}$ of \UAI{}$({\cal A},{\cal O},{\cal R},m,{\cal M},\rho)$ is defined as follows:
\begin{align}
 \MoveEqLeft Q(\masae_{1..k-1}, a_k)\nonumber\\
 &= \sum_{e_k\in {\cal E}} \max_{a_{k+1}\in{\cal A}} \sum_{e_{k+1}\in {\cal E}}   \max_{a_{k+2}\in{\cal A}} %
                          ... \max_{a_{m(k)}\in{\cal A}} \sum_{e_{m(k)}\in {\cal E}}  \nonumber\\
 &\quad
    \left(\res\soesum{{\cal M}}\rho(q)\right) \label{eq:UAIQ}
\end{align}
where
\[
 \dbl{q}(a_{1..m(k)}) = e_{1..m(k)}
\]
 denotes the condition that $q$ behaves interactively
 in the same way as in the case of Eq.~\eqref{eq:dblqAIXI},
 if $q$ is an interactive program
 such as that of monotone Turing machines and an implementation using lazy I/O.
${\cal M}$ can be a model of computation without interaction;
in such cases, $q'$ such that
\begin{align}
q'(a_{1..1}, []) &= e_1 \label{eq:qprime} \\ %
q'(a_{1..2}, e_{1..1}) &= e_2 \\
    \vdots \qquad  & \qquad \vdots \\
q'(a_{1..m(k)},e_{1..m(k)-1}) &= e_{m(k)}
\end{align}
has to be implemented, which requires recomputation of the states of the environment.
\footnote{We wrote $a_{1..1}$ and $e_{1..1}$ instead of $a_1$ and $e_1$ respectively,
because they are actually {\em lists of} actions and perceptions, i.e., $a_1:[]$ and $e_1:[]$.}

 Based on the action value function $Q$ defined above, \UAI{}$({\cal A},{\cal O},{\cal R},m,{\cal M},\rho)$ always chooses the best action, i.e., the action $\dot{a}_k$ at time $k$ after the actual interaction history $\dot{\masae}_{k-1}$ such that
\begin{equation}
 \dot{a}_k = \argmax_{a\in{\cal A}} Q(\dot{\masae}_{1..k-1}, a) \label{eq:UAIa}
\end{equation}

 \AIXI{} can be represented as \UAI{}$({\cal A},{\cal O},{\cal R},m,T,\xi)$ using the set of monotone Turing machines $T$ and the universal prior $\xi$.
 Although \AIXI{} does not explicitly limit the set of rewards to rational numbers unlike \UAI{}, how to represent real numbers is not discussed by papers on \AIXI{}.

\subsubsection{Computability of \UAI{}}
Is \UAI{} implementable using exact real arithmetic?
Our conclusion is that action values are \dwcable{}, and the only part that can cause an infinite loop is the $\arg\max$ operation, when using a \terminating{} as the environment model.
One may think that the infinite summation may also cause an infinite loop, but
 this is not the case actually %
if the set of rewards is bounded and does not change with time,
because the summation is bounded by a geometric series, and thus more and more digits become fixed from the most significant ones as the computation proceeds.

The following Theorem~\ref{th:UAIQ} clarifies the above claim.

\begin{theorem}[Computability of the action value function of \UAI{}]\label{th:UAIQ}
 The action value function
 $Q(\masae_{1..k-1}, a)$ of \UAI{}$({\cal A},{\cal O},{\cal R},m,{\cal M},\rho)$ defined by Eq.~\eqref{eq:UAIQ}
 is \dwcable{} if ${\cal R}\subset [0,\frac1{m(k)-k+1})$ holds and ${\cal M}$ is a \terminating{}.
\end{theorem}

See \aref{sec:proofsmain} for the proof of Theorem~\ref{th:UAIQ}.

This theorem shows that the action-value function of \UAI{} is exactly computable, or computable to arbitrary precision, provided that the set of possible environments is a set of total functions.
 We could not prove that the $\arg\max$ operation over the exact action values is computable for general priors such as the \AIXI{}'s $\xi$, because that operation involves comparisons between real numbers which can be exactly the same. The reader should notice that action selection of $Q$-learning also involves such $\arg\max$ operation over the real numbers. In other words, this theorem proves that \UAI{} is as computable as $Q$-learning.

Theorem~\ref{th:UAIQ} is proved for the case of ${\cal R}\subset [0,\frac1{m(k)-k+1})$; thanks to the following Lemma~\ref{lem:linearity}, this does not limit the applicability of the theorem.
\begin{lemma}[Linearity]\label{lem:linearity}
 For a positive real number $p$ and a real number $s$,
 let $Q'$ be the $Q$ value obtained by replacing $r$ in Eq.~\eqref{eq:UAIQ} with $r'$ where $r'(e) = p r(e) + s$, i.e.,
\begin{align}
\MoveEqLeft Q'(\masae_{1..k-1}, a_k)\nonumber\\
 &= \sum_{e_k\in {\cal E}} \max_{a_{k+1}\in{\cal A}} \sum_{e_{k+1}\in {\cal E}} \max_{a_{k+2}\in{\cal A}} ... \max_{a_{m(k)}\in{\cal A}} \sum_{e_{m(k)}\in {\cal E}} \nonumber\\
 &\quad \left( \left(\sum_{t=k}^{m(k)}r'(e_t)\right) \soesum{{\cal M}}\rho(q) \right) \label{eq:Qprime}
\end{align}
Then, for a real number $c(\masae_{1..k-1})$,
\begin{equation}
 Q'(\masae_{1..k-1}, a) = p Q(\masae_{1..k-1}, a) + sc(\masae_{1..k-1}) \label{eq:linearQprime}
\end{equation}
 holds.
\end{lemma}

See \aref{sec:proofsmain} for the proof of Lemma~\ref{lem:linearity}.

\begin{corollary}\label{cor:linearity}
   For any positive real number $p$ and any real number $s$,
  the value of Eq.~\eqref{eq:UAIa} does not change when $r$ in Eq.~\eqref{eq:UAIQ} is replaced with $r'$ where $r'(e) = p r(e) + s$. In other words,
  for $Q'$ defined by Eq.~\eqref{eq:Qprime}
 \[
 \argmax_{a\in{\cal A}} Q(\masae_{1..k-1}, a) = \argmax_{a\in{\cal A}} Q'(\masae_{1..k-1}, a)
 \]
\end{corollary}
\begin{proof}
 Self-explanatory from Lemma~\ref{lem:linearity}.
\end{proof}

Thanks to lazy evaluation, \UAI{} automatically omits unnecessary executions of environment candidate programs which do not affect the result of the $\arg\max$ operation.
How much computation time is saved is unpredictable, depending on the history.
However, \UAI{} can return the result in the precision it has at the deadline, if there exists the deadline for each interaction step.
It can have the $\arg\max$ set based on the first digit of $Q$ values, the $\arg\max$ set based on the first two digits of $Q$ values, and so on, and
randomly select from the most precise $\arg\max$ set at the deadline.

\subsection{Making \UAI{} Fully Computable}

  Now we show that even the $\arg\max$ part of \UAI{} can be made implementable by modifying the prior distribution adequately.
 We call the resulting AI model {\em Unlimited Computable AI (\UCAI{})}.
 
 Because AIXI assigns less plausibility to longer programs, it is possible (and in fact, it is really the case) that 
 only a finite number of the most plausible programs affect the result of the decision making by the $\arg\max$ operator over the action-values,
 and the remaining infinitely many programs do not affect it.
 The action-values need to be computed only to the precision where we can tell the difference between them.
 All we have to take care of is not to compare exactly the same values, because the comparison between the same values to the precision where they make difference results in an infinite loop.

Thus, we can concentrate on avoiding comparisons between action values which may be the same.
We can do two things:
\begin{itemize}
 \item compare two values that are known to be different beforehand, and tell which is the greater;
 \item skip comparison of two values that are known to be the same beforehand, and say that they are the same.
\end{itemize}
In other words, it is enough to know whether two action values are the same or not before comparison. 

 Accidental coincidences between action values may be avoided by
\begin{itemize}
  \item slightly modifying the lower bits of prior distribution to consist of irrational numbers,
  \item limiting ${\cal R}$ to positive numbers in order to avoid the sum of rewards happening to be 0.
\end{itemize}

\subsubsection{Priors for Making Difference in Values}

This section discusses how to assign prior probabilities in order to make the whole things computable.

   Our idea
  is to subtract %
 small
  randomized
  number $\delta d(\eta(q))$ from the normal rational prior for each $q$, where $\eta(q)$ denotes the natural number
  representation
  of $q$.
The resulting prior $\xi_2(q)$ can be represented as
\begin{equation}
\xi_2(q) = \xi_1(q)(1 - \delta d(\eta(q))) \label{eq:rhotwo}
\end{equation}
by using $\xi_1$ of
  Eq.~\eqref{eq:generalPrior}.
 
 $\eta$ can be defined as adding $1$ to the left of reversed $q$ and interpreting the result as a binary number, i.e.,
 \begin{align}
  \eta([])  &= 1 \\
  \eta(b:x) &= b + 2\eta(x)
 \end{align}
 if $q$ is defined as a list of bits, or $\{0,1\}*$.
 Then,
 \begin{align}
  l(q) &= \lfloor\log_2 \eta(q)\rfloor 
 \end{align}
 and thus
 \begin{align}
  \xi_1(q) &= 2^{-\lfloor\log_2 \eta(q)\rfloor} \\
  \xi_2(q) &= 2^{-\lfloor\log_2 \eta(q)\rfloor}(1 - \delta d(\eta(q)))
 \end{align}

  $\delta$ is a positive small rational number less than $1$; a reasonable choice is $2^{-64}$, which makes the additional term 
 insignificant %
 for those who
 do not care about %
 the difference between fixed precision floating point approximations and real numbers.

  The function $d$ assigns a different infinite binary fraction randomly for each natural number $i$.
 $d(i)$ can be obtained by splitting
 an ideally %
 random bit stream source for $i$ times:
  \begin{align}
 d(i) &= \dbl{d'(i)} \\
 d'(i) &= \tau_0(split(\underbrace{\tau_1(split( ... \tau_1(split(}_i\langle\text{random bit stream}\rangle)) ... )))
 \end{align}
  where $\tau_0$ and $\tau_1$ are projection functions
  \begin{align}
 \tau_0(x,y) &= x, & \tau_1(x,y) &= y
 \end{align}
  and $split \in \{0,1\} \rightarrow (\{0,1\} \times \{0,1\})$ is a function that splits a random stream into two random streams.
   Ideal acyclic random streams can be split by the leap-frog method, 
 though practically, pseudo-random number generators must be split carefully.\citep[e.g.,][]{tfrandom}
 In this paper, we show the computability of \UCAI{} with probability 1,
 assuming that an ideal random bit stream source is available.

Now we can prove the following theorem:
\begin{theorem}[Computability of \UCAI{}]\label{th:UCAI}
 Let ${\cal M}$ be a model of computation that only includes terminating programs and has a conditional construct. %
 Then, \UAI{}$({\cal A},{\cal O},{\cal R},m,{\cal M},\xi_2)$ is computable
  with probability 1 for $\xi_2$ in Eq.~\eqref{eq:rhotwo}.
 \end{theorem}
See \aref{sec:proofsmain} for the proof of Theorem~\ref{th:UCAI}.

\section{Conclusions}\label{sec:conclusions}
This paper proposed \AIXI{} variants supporting a broader class of the environment than \AIXItl{}, and proved
theorems on
their computability.

When considering the real-world interaction, the processing time for each interaction step should be considered limited, even if we permit the discrete-time model.
In this sense, a timeout is a natural idea, and it is understandable to limit the program length, considering that the information accessible within a limited time is limited.
However, considering that the real-world which has %
vast space is highly parallel, \AIXItl{} which models the environment using a sequential model of computation is not necessarily the best selection. 

On the other hand, \UCAI{} needs to simulate the environment, many times at each interaction step.
It is unnatural to think that a \UCAI{} agent as a computer is as parallel as the environment.

 Still, we think \UCAI{} is more powerful than \AIXItl{} even when there is a deadline at each interaction step.
 In the case of \AIXItl{}, each environment candidate program timeouts at time $\tilde{t}$, and there are $2^{\tilde{l}}$ candidate programs.
 In the case of \UCAI{}, on the other hand, there is no timeout for each environment candidate, and a \UCAI{} agent only needs to timeout at the actual deadline $t'$, and thus
 more programs can be tried.
The main advantage of \UCAI{} comes from the fact that \UCAI{} automatically omits execution of longer environment candidate programs when shorter ones cause a difference in the action values, thanks to lazy evaluation.

One question is whether we need to use a Turing complete model for modeling the behavior of the environment at each interaction step,
i.e., whether there is a case where \AIXI{} should be used over \UCAI{}.
We think \UCAI{} is enough for the following reasons:
\begin{enumerate}
 \item indeed, this world can simulate Turing machines with finite tape, but requesting them to terminate within one interaction step is not a reasonable idea;
     we still request that the environment can simulate Turing machines by holding physical tapes;\label{item:tapeInEnvironment}
 \item \AIXI{} does not adequately model the real world because it permits incomputable agents.
\end{enumerate}
 For the purpose of satisfying Item~\ref{item:tapeInEnvironment} only,
 it is enough if at each interaction step the environment model can compute the finite-to-finite map implementing the set of quadruples of Turing machines, even with a tabular representation without loops.
 \MCAIXI{} \citep{MCAIXI},
 which approximates \AIXI{} more aggressively than \AIXItl{},
 does not model the environment as a program, but instead, it uses Context Tree Weighting (CTW) \citep{CTW}, that is essentially the tabular representation using the PATRICIA tree.\citep{Katayama16} 
We should note that tables cannot generalize, though. %

 \bibliographystyle{plainnat}
 \bibliography{skatayama}

\begin{thebibliography}{10}
\providecommand{\natexlab}[1]{#1}
\providecommand{\url}[1]{\texttt{#1}}
\expandafter\ifx\csname urlstyle\endcsname\relax
  \providecommand{\doi}[1]{doi: #1}\else
  \providecommand{\doi}{doi: \begingroup \urlstyle{rm}\Url}\fi

\bibitem[Bird and Wadler(1988)]{BirdWadler}
Richard Bird and Philip Wadler.
\newblock \emph{An Introduction to Functional Programming}.
\newblock Prentice-Hall, 1988.

\bibitem[Boehm and Cartwright(1990)]{BoehmCartwright1993}
Hans Boehm and Robert Cartwright.
\newblock Exact real arithmetic formulating real numbers as functions.
\newblock In David~A. Turner, editor, \emph{Research Topics in Functional
  Programming}, pages 43--64. Addison-Wesley Longman Publishing Co., Inc.,
  Boston, MA, USA, 1990.
\newblock ISBN 0-201-17236-4.

\bibitem[Claessen and Pa{\l}ka(2013)]{tfrandom}
Koen Claessen and Micha\l~H. Pa{\l}ka.
\newblock Splittable pseudorandom number generators using cryptographic
  hashing.
\newblock In \emph{Proceedings of the 2013 ACM SIGPLAN Symposium on Haskell},
  Haskell '13, pages 47--58, New York, NY, USA, 2013. ACM.
\newblock ISBN 978-1-4503-2383-3.
\newblock \doi{10.1145/2503778.2503784}.
\newblock URL \url{http://doi.acm.org/10.1145/2503778.2503784}.

\bibitem[Hutter(2007)]{topdown}
Marcus Hutter.
\newblock Universal algorithmic intelligence: A mathematical
  top$\rightarrow$down approach.
\newblock In B.~Goertzel and C.~Pennachin, editors, \emph{Artificial General
  Intelligence}, Cognitive Technologies, pages 227--290. Springer, Berlin,
  2007.
\newblock ISBN 3-540-23733-X.
\newblock URL \url{http://www.hutter1.net/ai/aixigentle.htm}.

\bibitem[Katayama(2016)]{Katayama16}
Susumu Katayama.
\newblock Ideas for a reinforcement learning algorithm that learns programs.
\newblock In \emph{Artificial General Intelligence - 9th International
  Conference, {AGI} 2016, {AGI} 2016, New York, USA, July 16--19, 2016,
  Proceedings}, pages 354--362, 2016.

\bibitem[Plume(1998)]{Plume98}
Dave Plume.
\newblock \emph{A Calculator for Exact Real Number Computation}.
\newblock PhD thesis, University of Edinburgh, 1998.

\bibitem[Sutton and Barto(1998)]{SuttonBarto}
Richard~S. Sutton and Andrew~G. Barto.
\newblock \emph{Introduction to Reinforcement Learning}.
\newblock MIT Press, Cambridge, MA, USA, 1st edition, 1998.
\newblock ISBN 0262193981.

\bibitem[Veness et~al.(2011)Veness, Ng, Hutter, Uther, and Silver]{MCAIXI}
Joel Veness, Kee~Siong Ng, Marcus Hutter, William Uther, and David Silver.
\newblock A {Monte-Carlo} {AIXI} approximation.
\newblock \emph{Journal of Artificial Intelligence Research}, 40:\penalty0
  95--142, 2011.

\bibitem[Watkins and Dayan(1992)]{WD92}
Christopher J. C.~H. Watkins and Peter Dayan.
\newblock {$Q$}-learning.
\newblock In \emph{Machine Learning}, pages 279--292, 1992.

\bibitem[Willems et~al.(1995)Willems, Shtarkov, and Tjalkens]{CTW}
Frans M.~J. Willems, Yuri~M. Shtarkov, and Tjalling~J. Tjalkens.
\newblock The context tree weighting method: Basic properties.
\newblock \emph{IEEE Transactions on Information Theory}, 41:\penalty0
  653--664, 1995.

\end{thebibliography}

\appendix

 \section{The Detailed Proofs}
 This section gives the proofs of the theorems.
 In \aref{sec:smalllemmas} we prove several utility lemmas on digit-wise computabilities of simple operators on FPRBSR's.
 In \aref{sec:biglemmas} we prove other lemmas that are necessary for proving the main theorems.
 Then, in \aref{sec:proofsmain} the main theorems are proved.

\subsection{Digit-wise Computabilities of Operations on FPRBSR's}\label{sec:smalllemmas}

Because the sum of two FPRBSR's may overflow, firstly we %
show that their average is \dwcable{}.
\begin{lemma}[Average is \dwcable{}]\label{lemma:ave}
 If $\xs$ and $\ys$ are both \dwcable{} FPRBSR's, then their average $\dbl{\xs} \oplus \dbl{\ys} = \frac{\dbl{\xs} + \dbl{\ys}}2$ is also \dwcable{} as an FPRBSR.
\end{lemma}
\begin{proof}
 The following algorithm $\opluss$ computes $\dbl{\xs} \oplus \dbl{\ys}$ as an FPRBSR:
 \begin{align}
   \MoveEqLeft (x_1:x_2:x_{3\infty}) \opluss (y_1:y_2:y_{3\infty}) \\
   &= \left\{
    \begin{array}{l l}
     \frac{x_1+y_1}2 : c     : s, & \mbox{ if $x_1+y_1 \in \{-2,0,2\}$} \\
     -1              : c + 1 : s, & \mbox{ if $x_1+y_1 = -1$ and $x_2+y_2 < 0$} \\
     0               : c - 1 : s, & \mbox{ if $x_1+y_1 = -1$ and $x_2+y_2 \ge 0$} \\
     0               : c + 1 : s, & \mbox{ if $x_1+y_1 =  1$ and $x_2+y_2 < 0$} \\
     1               : c - 1 : s, & \mbox{ if $x_1+y_1 =  1$ and $x_2+y_2 \ge 0$}
    \end{array}
   \right. \label{eq:ave}
 \end{align}
 where
 $c : s = (x_2:x_{3\infty}) \opluss (y_2:y_{3\infty})$.

The idea of Eq.~\eqref{eq:ave} is as follows:
\begin{itemize}
 \item if $x_1+x_2$ is even, then the average of $x_1$ and $y_1$ does not affect %
     less significant digits, and thus they can be computed straightforwardly; %
 \item if $x_1+x_2=-1$, then $1/4$ must be subtracted from the result; %
     this can be achieved by either subtracting $1$ from the second element of the result or subtracting $1$ from the first element and adding $1$ to the second element;
     the algorithm conditions on the sign of $x_2+y_2$ in order to keep each digit within $\{-1,0,1\}$ without carrying out;
 \item if $x_1+x_2=1$, then $1/4$ must be added to the result; %
     this can be achieved by either adding $1$ to the second element of the result, or adding $1$ to the first element and subtracting $1$ from the second element;
     again, the algorithm conditions on the sign of $x_2+y_2$ in order to keep each digit within $\{-1,0,1\}$ without carrying out;
\end{itemize}

  $\opluss$ defined above is \dwcable{}, because each recursive call of $\opluss$ determines one digit.
\end{proof}

\begin{lemma}[Sum is \dwcable{}]\label{lemma:add}
 If $\xs$ and $\ys$ are both \dwcable{} FPRBSR's and their average $\dbl{\xs}\oplus\dbl{\ys}$ is within $(-1/2, 1/2)$, then their sum $\dbl{\xs} + \dbl{\ys}$ is also \dwcable{} as an FPRBSR.
\end{lemma}
\begin{proof}
FPRBSR's within $(-1/2, 1/2)$ can be doubled by the following \dwcable{} function $double$.
\begin{eqnarray*}
 double( 0:x_{2\infty}) &=& x_{2\infty} \\
 double( 1:x_{2\infty}) &=& f(x_{2\infty}) \\
 double(-1:x_{2\infty}) &=& g(x_{2\infty}) \\
 f( 0:x_{2\infty}) &=&  1:f(x_{2\infty}) \\
 f(-1:x_{2\infty}) &=&  1:x_{2\infty}\\
 g( 0:x_{2\infty}) &=& -1:g(x_{2\infty}) \\
 g( 1:x_{2\infty}) &=& -1:x_{2\infty}
\end{eqnarray*}

The above functions implement the following real computations for $x_\infty$, $y_\infty$, and $z_\infty$ such that $\dbl{x_\infty}\in(-1/2, 1/2)$, $\dbl{y_\infty}\in[-1,0)$, and $\dbl{z_\infty}\in(0,1]$:
\begin{align}
\dbl{double(x_{\infty})} &= 2 \dbl{x_{\infty}} \\
\dbl{f(y_{\infty})}      &= 1 + \dbl{y_{\infty}} \\
\dbl{g(z_{\infty})}      &= -1 + \dbl{z_{\infty}}
\end{align}

Thus, $\dbl{\xs} + \dbl{\ys}$ can be computed as $\xs +_\infty \ys = double(\xs \opluss \ys)$.
\end{proof}

 Likewise, multiplication of redundant binary representations is known to be computable.\citep{Plume98}

\begin{lemma}[Product is \dwcable{}]\label{lemma:mul}
 If $\xs$ and $\ys$ are both \dwcable{} FPRBSR's, then their product $\dbl{\xs}\dbl{\ys}$ is also \dwcable{} as an FPRBSR.
\end{lemma}

We also need to show that
the maximum of two FPRBSR's %
is \dwcable{}.

 For $x$ and $y$ in the non-redundant binary representation, $\max\{x,y\}$ can be computed by adopting digits of whichever of $x$ and $y$ until we know which is greater, and then use the digits from the greater of the two.
 This algorithm does not work correctly for $x$ and $y$ in the redundant binary representation, because we cannot tell which is greater by only comparing digits (e.g., $\dbl{1:-1:1:\xs} = \dbl{1:0:-1:\xs} = \dbl{0:1:1:\xs}$).

Our solution uses the fact that
\[
\max\{x,y\} = \frac{x + y + |x-y|}2= \frac{x + y + |x+(-y)|}2
\]
We need to show that negation and taking the absolute value are \dwcable{}.

\begin{lemma}[Negation is \dwcable{}]\label{lemma:neg}
 If $\xs$ is a \dwcable{} FPRBSR, then its negation $-\dbl{\xs}$ is also \dwcable{} as an FPRBSR.
\end{lemma}
\begin{proof}
 The following algorithm $-_\infty(\cdot)$ computes $-\dbl{\xs}$ as an FPRBSR:
 \begin{eqnarray*}
   -_\infty(y:\ys) &=& -y : -_\infty(\ys)
 \end{eqnarray*}
\end{proof}

\begin{lemma}[The absolute value is \dwcable{}]\label{lemma:abs}
 If $\xs$ is a \dwcable{} FPRBSR, then its absolute value $|\dbl{\xs}|$ is also \dwcable{} as an FPRBSR.
\end{lemma}
\begin{proof}
 The following algorithm $|\cdot|_\infty$ computes $|\dbl{\xs}|$ as an FPRBSR:
 \begin{eqnarray*}
   |0:\ys|_\infty &=& 0 : |\ys|_\infty \\
   |1:\ys|_\infty &=& 1 : \ys \\
   |-1:\ys|_\infty &=& 1 : -_\infty(\ys)
 \end{eqnarray*}
\end{proof}

\begin{lemma}[The binary $\max$ is \dwcable{}]\label{lemma:binaryMax}
 If $\xs$ and $\ys$ are both \dwcable{} FPRBSR's, then their maximal value $\max\{\dbl{\xs},\dbl{\ys}\}$ is also \dwcable{} as an FPRBSR.
\end{lemma}
\begin{proof}
Since
\[
\frac{\max\{x,y\}}2 = (x\oplus y) \oplus |x\oplus(-y)|
\]
$\frac{\max\{x,y\}}2$ is \dwcable{}.

Since $\max\{x,y\}$ is either $x$ or $y$, it can be represented in FPRBSR, and
$\frac{\max\{x,y\}}2 \in (-1/2, 1/2)$. %
Therefore, $\max\{\dbl{\xs},\dbl{\ys}\}$ can be computed as $(\xs \opluss \ys) +_\infty |\xs \opluss -_\infty(\ys) |_\infty$.
\end{proof}

\begin{lemma}[$\max$ over sets is \dwcable{}.]\label{lemma:max}
 If $X$ is a finite non-empty set of \dwcable{} FPRBSR's, then $\max X$ is also a \dwcable{} FPRBSR.
\end{lemma}
\begin{proof}
 Let $\xs$ and $\ys$ be FPRBSR's.

 If $X$ is a finite and non-empty set of \dwcable{} FPRBSR's, $\max X$ can be computed by a finite number of binary $\max$ operations in the following way:
 \begin{align}
   \max \{ \xs \}    &= \xs \\
   \max (\{ \xs, \ys \} + Y) &= \max (\{\max \{ \xs, \ys \}\} + Y)
 \end{align}
 where $+$ over sets denotes the direct sum of two sets.
\end{proof}

\begin{definition}[Comparison]
 {\em Comparison} of two values is either $(<)$, $(\equiv)$, or $(>)$,
 where `values' can be digits, FPRBSR's, or tuples of them.
 Comparison $\cmp$ (except that of FPRBSR's) is defined in the following way, using the usual order relations $<$ and $>$ and the equality relation $=$:

 \begin{equation}
  x \cmp y = \left\{
   \begin{array}{l l}
       (<),      \mbox{ if $x<y$} \\
       (\equiv), \mbox{ if $x=y$} \\
       (>),      \mbox{ if $x>y$}
   \end{array}
   \right.
 \end{equation}
\end{definition}

\begin{lemma}[Comparison of two different reals is computable]\label{lemma:comparison}
 If $\xs$ and $\ys$ are both \dwcable{} FPRBSR's and we know $\dbl{\xs}\ne\dbl{\ys}$ beforehand, then, their comparison $\dbl{\xs} \cmp \dbl{\ys}$ is computable.
\end{lemma}
\begin{proof}
 $\dbl{\xs}\oplus(-\dbl{\ys})$ is \dwcable{}.
 The comparison can be computed as $\xs \cmp \ys = cmp(\xs \opluss -_\infty(\ys))$ using $cmp$ defined as follows:
 \begin{align}
  cmp(0:\zs) &= cmp(\zs)\\
  cmp(-1:\zs) &= (<) \\
  cmp(1:\zs) &= (>)
 \end{align}
\end{proof}

\begin{lemma}[$\arg\max$ of a monomorphism over a finite set is computable.]\label{lemma:argmax}
 If $X$ is a finite non-empty set and $f \in X \rightarrow R$ is a monomorphism which is \dwcable{} using FPRBSR, then, $\argmax_{x\in X} f(x)$ is computable.
\end{lemma}
\begin{proof}
 Let $f_\infty \in X \rightarrow \{-1,0,1\}\infty$ be a function that computes $f$, i.e., $\forall x\in X. \dbl{f_\infty(x)} = f(x)$.

 $\argmax_{x\in X} f(x)$ can be computed by the following algorithm:
 \begin{align*}
     \argmax_{x\in \{a\}}    f(x) &= a \\
     \argmax_{x\in \{a,b\}+Y} f(x) &= \left\{
         \begin{array}{l l}
             \argmax_{x\in \{a\}+Y} f(x), & \mbox{ if $f_\infty(a) \cmp f_\infty(b) = (>)$;} \\
             \argmax_{x\in \{b\}+Y} f(x), & \mbox{ if $f_\infty(a) \cmp f_\infty(b) = (<)$.}
         \end{array}
         \right.
 \end{align*}
\end{proof}

\subsection{digit-wise computability of Diminishing Series}\label{sec:biglemmas}
In this section, we provide lemmas and their proofs about digit-wise computability of infinite series of positive values diminishing exponentially.
 Their main purpose is to be applied to the infinite summation in Eq.~\eqref{eq:UAIQ}.

\begin{lemma}[Diminishing series by $2^{-2}$ are \dwcable{}.]\label{lem:diminishing}
 If $0\le \dbl{x_\infty(i)}<1$ holds and $x_\infty(i)$ is a \dwcable{} FPRBSR for all positive integer $i$, then, $\sum_{i=1}^\infty 2^{-2i}\dbl{x_\infty(i)}$ is \dwcable{}.
\end{lemma}

\begin{proof}
We show that $\sum_{i=1}^\infty 2^{-2i}\dbl{x_\infty(i)}$ is \dwcable{} as $0:S(1)$ using the function $S$ defined as follows:
\begin{align}
S(i) &= \left\{\begin{array}{l l}
        1:((-1:\sss) +_\infty 0:S(i+1)), & \text{ if $s = 1$,} \\
        0:((s :\sss) +_\infty 0:S(i+1)), & \text{ if $s\in\{-1,0\}$}
        \end{array}\right. \label{eq:S1s} \\
     &\text{where $s:\sss = \xs(i)$.} \nonumber
\end{align}
Note that $s=-1$ suggests $\sss = 1:\sss$ because $\dbl{\xs(i)}\ge0$.
By definition, $\dbl{S(i)} \ge 0$ holds. 
Obviously, $S$ is \dwcable{}.

Now we show that
\[
\sum_{i=1}^\infty 2^{-2i}\dbl{x_\infty(i)} = \dbl{0:S(1)}
\]

When $\xs(i) = 1:\sss$, from the first case of Eq.~\eqref{eq:S1s}, we obtain
\begin{align}
\dbl{S(i)} &= \dbl{1:((-1:\sss) +_\infty 0:S(i+1))} \nonumber\\
           &= 2^{-1} (1 + \dbl{(-1:\sss) +_\infty 0:S(i+1)}) \nonumber\\
           &= 2^{-1} (1 + \dbl{(-1:\sss)} + \dbl{0:S(i+1)}) \nonumber\\
           &=         2^{-1}(\dbl{ (1:\sss)} + \dbl{0:S(i+1)}) \nonumber\\
           &=         2^{-1}\dbl{\xs(i)}   + 2^{-2}\dbl{S(i+1)} \label{eq:dblS1}
\end{align}
Likewise, for $\xs(i) = s:\sss$ where $s\in\{-1,0\}$ from the second case of Eq.~\eqref{eq:S1s},
\begin{align}
\dbl{S(i)} &= \dbl{0:((s:\sss) +_\infty 0:S(i+1))} \nonumber\\
           &= 2^{-1}\dbl{\xs(i)} + 2^{-2}\dbl{S(i+1)} \label{eq:dblSs}
\end{align}
For both cases, from Eqs.~\eqref{eq:dblS1} and \eqref{eq:dblSs}, %
\begin{align}
\dbl{S(i)} &= 2^{-1}\dbl{\xs} + 2^{-2}\dbl{S(i+1)} \label{eq:dblS}
\end{align}

Thus, by applying Eq.~\eqref{eq:dblS} repeatedly,
\begin{align*}
  \MoveEqLeft \dbl{0:S(1)} \\ %
                            &= 2^{-2}\left(\dbl{\xs(1)} + 2^{-2}\left(\dbl{\xs(2)} + 2^{-2}\left(\dbl{\xs(3)} + ...\right)\right)\right) \\
                            &= \sum_{i=1}^\infty 2^{-2i}\dbl{\xs(i)}
\end{align*}

\end{proof}

\begin{lemma}\label{lem:sumrho}
 Let $\rho : {\cal X}{*} \rightarrow [0,1]$ a \dwcable{} probability function over finite lists of ${\cal X}$'s.
 Also, let $P$ a computable predicate over such lists.
 Then, $\sum_{p:P(p)}\rho(p)$ is \dwcable{}.
\end{lemma}
\begin{proof}
 Since $\rho$ is a probability function,
 \[
  \sum_{p} \rho(p) = 1
 \]
 holds. Thus,
 \[
  \sum_{p:P(p)} \rho(p) \le 1
 \]
 By reorganizing the summation from the shortest increasing the length, the \lhs{} can be rewritten to
 \begin{align}
  \sum_{k=1}^\infty \sum_{p:l(p)=k,P(p)} \rho(p) &\le 1 \\
  \sum_{k=1}^\infty \sum_{p:l(p)=k} \rho(p) &= 1
 \end{align}
 where $l(p)$ denotes the length of $p$.

 Now let
\begin{align}
    R(i) &= \sum_{p:l(p)=i,P(p)} \rho(p) \\
    R'(i) &= \sum_{p:l(p)=i} \rho(p)
\end{align}
 Then,
 \begin{align}
  R(i) &\le R'(i) \\
  \sum_{k=1}^\infty R(k) &\le 1 \\
  \sum_{k=1}^\infty R'(k) &= 1 \label{eq:sumRprime}
 \end{align}
 and $R(i)$ and $R'(i)$ are \dwcable{} because they consist of finite summations and \dwcable{} computations.
 Let $R_\infty(i)$ and $R_\infty'(i)$ be FPRBSR representations of $R(i)$ and $R'(i)$ respectively.

 Because the \lhs{} of Eq.~\eqref{eq:sumRprime} is well-defined,
\[
\lim_{i\rightarrow\infty} \sum_{k=i}^\infty R'(k) = 0
\]
 holds.
In other words, for any $\epsilon>0$ there exists a natural number $n>0$ that satisfies $\forall j>n. \sum_{k=j}^\infty R'(k)<\epsilon$.
Thus, for some Skolem function $f$ that takes $\epsilon>0$ and returns such a natural number $n>0$,
\[
\sum_{k=f(2^{-2i})}^\infty R'(k) < 2^{-2i}
\]
holds.

There are infinite candidates for $f$ because we need not choose the minimal $n$.
Let
\[
 g(i) = f(2^{-2i})
\]
Then, we can choose the following computable implementation of $g$:
\begin{align}
 (g(0), \sss(0)) &= (1, 0_\infty) \\
 (g(i), \sss(i+1)) &= g'(i, g(i), \sss(i)), i\ge0 \\
 g'(i, n, 0:\sss') &= g'(i, n+1, 0:\sss' +_\infty R_\infty'(n)) \\   %
 g'(i, n, 1:\sss') &= \left\{
                      \begin{array}{l}
                       (n+1, 1:\sss'), \qquad\hfill \mbox{ if $\sss'$ is prefixed with $2i$ 0's;} \\
                       g'(i, n+1, 1:\sss' +_\infty R_\infty'(n)), \hfill \mbox{ otherwise.}  %
                      \end{array}\right.
\end{align}
where
\[
 0_\infty = 0 : 0_\infty
\]

If we define $x$ as
\[
 x(k) = \sum_{i=g(k)}^{g(k+1) - 1}R(i)
\]
then, $x$ is a sequence that diminishes by the rate of $2^{-2}$.
Moreover, each $x(k)$ is \dwcable{} because it can be computed from finite times of additions.
Therefore, from Lemma~\ref{lem:diminishing},
\begin{align*}
 \MoveEqLeft \sum_{i=1}^\infty R(i) \\
 &= \sum_{i=1}^{g(1) - 1}R(i) + \sum_{i=g(1)}^{g(2) - 1}R(i) + \sum_{i=g(2)}^{g(3) - 1}R(i)+ ... + \sum_{i=g(k)}^{g(k+1) - 1}R(i)+ ... \\
 &= \sum_{i=1}^{g(1) - 1}R(i) + x(1) + x(2) + ...
\end{align*}
is \dwcable{} because it is a series diminishing by $2^{-2}$.
\end{proof}

\subsection{Proofs of the Main Theorems}\label{sec:proofsmain}

\begin{proof}[Proof of Lemma~\ref{lem:linearity}]
 \begin{align}
     \MoveEqLeft Q'(\masae_{1..k-1}, a_k)\nonumber\\
     &= \sum_{e_k\in {\cal E}} \max_{a_{k+1}\in{\cal A}} \sum_{e_{k+1}\in {\cal E}} \max_{a_{k+2}\in{\cal A}} ... \max_{a_{m(k)}\in{\cal A}} \sum_{e_{m(k)}\in {\cal E}} \nonumber\\
     &\quad \left(\left(p\res + s(m(k)-k+1)\right)\soesum{{\cal M}}\rho(q)\right) \\
     &= \sum_{e_k\in {\cal E}} \max_{a_{k+1}\in{\cal A}} \sum_{e_{k+1}\in {\cal E}} \max_{a_{k+2}\in{\cal A}} ... \max_{a_{m(k)}\in{\cal A}} \sum_{e_{m(k)}\in {\cal E}} \nonumber\\
     &\quad 
         \begin{array}{l l}
         \bigg( p \res\soesum{{\cal M}}\rho(q) \\
         + s (m(k)-k+1) \soesum{{\cal M}}\rho(q) \bigg)
         \end{array}
\end{align}

Now, for each $a_{m(k)}$ there always exists only one $e_{m(k)}$, because ${\cal M}$ is a \terminating{}. Thus,
\[
\sum_{e_{m(k)}\in {\cal E}}  \soesum{{\cal M}}\rho(q)
 = \soes{m(k)-1}{{\cal M}}\rho(q)
\]
More generally,
\[
\sum_{e_{t+1}\in {\cal E}}  \soes{t+1}{{\cal M}}\rho(q)
 = \soes{t}{{\cal M}}\rho(q)
\]

Thus,
\begin{align}
     \MoveEqLeft Q'(\masae_{1..k-1}, a_k)\nonumber\\
     &= p \left(
            \sum_{e_k\in {\cal E}} \max_{a_{k+1}\in{\cal A}} \sum_{e_{k+1}\in {\cal E}} \max_{a_{k+2}\in{\cal A}} ... \max_{a_{m(k)}\in{\cal A}} \sum_{e_{m(k)}\in {\cal E}} \right.\nonumber\\
     &\qquad \qquad \left.\left( \res\soesum{{\cal M}}\rho(q) \right) \right) \nonumber\\
     &\qquad \qquad
           + s (m(k)-k+1) \sum\limits_{\substack{
                                 q\in{\cal M},\\
                                 \dbl{q}(a_{1..k-1}) = e_{1..k-1}
                               }}\rho(q) \\
     &= p Q(\masae_{1..k-1}, a_k)
           + s (m(k)-k+1) \sum\limits_{\substack{
                                 q\in{\cal M},\\
                                 \dbl{q}(a_{1..k-1}) = e_{1..k-1}
                               }}\rho(q)
\end{align}
Therefore,
Eq.~\eqref{eq:linearQprime} holds for 
\[
 c = (m(k)-k+1) \sum\limits_{\substack{
                                 q\in{\cal M},\\
                                 \dbl{q}(a_{1..k-1}) = e_{1..k-1}
                               }}\rho(q)
\]
\end{proof}

\begin{proof}[Proof of Theorem~\ref{th:UAIQ}]

    From Lemma~\ref{lem:sumrho},
    \begin{equation}
        \soesum{{\cal M}}\rho(q) \label{eq:sumrhoQ}
    \end{equation}
    is computable if ${\cal M}$ is a \terminating{}.

    From Lemmas~\ref{lemma:add}, \ref{lemma:mul}, and \ref{lemma:max}, $Q(\masae_{1..k-1}, a)$ defined by Eq.~\eqref{eq:UAIQ} is computable,
    because the \rhs{} of Eq.~\eqref{eq:UAIQ} only consists of addition, multiplication, maximization, and the \rhs{} of \eqref{eq:sumrhoQ}.
\end{proof}

\begin{proof}[Proof of Theorem~\ref{th:UCAI}]
 From Corollary~\ref{cor:linearity} we can obtain an equivalent \UAI{} algorithm satisfying ${\cal R}\subset (0,\frac1{m(k)-k+1})$ if we know the lower bound and the upper bound of the set of rewards.%
 \footnote{Note that ${\cal R}$ must not include $0$ unlike Theorem~\ref{th:UAIQ}.}
 From Theorem~\ref{th:UAIQ}, the action value function in the $\arg\max$ operation in Eq.~\eqref{eq:UAIa} is \dwcable{}.
 Since $\arg\max$ for any monomorphic function is computable from Lemma~\ref{lemma:argmax}, it is enough to show that
\begin{equation}
 a \mapsto Q(\masae_{1..k-1}, a) \label{eq:Qona}
\end{equation}
is monomorphic for any $\masae_{1..k-1}$
  with probability 1.

 From Eq.~\eqref{eq:UAIQ}, we obtain
 \begin{align}
  \MoveEqLeft Q(\masae_{1..k-1}, a_k)\nonumber\\
  &= \sum_{e_k\in {\cal E}} \max_{a_{k+1}\in{\cal A}} \sum_{e_{k+1}\in {\cal E}} \max_{a_{k+2}\in{\cal A}} ... \max_{a_{m(k)}\in{\cal A}} \sum_{e_{m(k)}\in {\cal E}} \nonumber\\
  &\quad    \left(\res\soesum{{\cal M}}\xi_2(q)\right) \\ %
  &= \max_{a_{k+1}\in{\cal A}} \max_{a_{k+2}\in{\cal A}} ... \max_{a_{m(k)}\in{\cal A}}
      \left(
      \sum\limits_{\substack{
                                 q\in {{\cal M}},\\
                                 \dbl{q}(a_{1..k-1}) = e_{1..k-1}
                               }}
            \left(\sum_{t=k}^{m(k)}r(e_t)\right) \xi_2(q)\right) \quad \label{eq:kore}  %
 \end{align}
  where $e_t$ in Eq.~\eqref{eq:kore} for $t\ge k$ is obtained from $e_{1..m(k)} = \dbl{q}(a_{1..m(k)})$.
 Note that $\sum_{t=k}^{m(k)}r(e_t)$ is positive and rational because $m(k)\ge k$ and ${\cal R}\subset (0,\frac1{m(k)-k+1}) \cap \mathds{Q}$.

 Now, since from the premise ${\cal M}$ is equipped with a conditional construct, %
 we can consider the following environment program candidate $q(b)$ for each $b\in{\cal A}$:
 \begin{align*}
  \MoveEqLeft \dbl{q(b)}(a_{1..i}) \\
  &= \left\{\begin{array}{l l}
          e_{1..i}, & \mbox{if $i<k$;} \\
          e_{1..k-1} \append replicate(m(k)-k+1, (o, r_{\max})), & \mbox{if $i\ge k$ and $a_k = b$;} \\
          e_{1..k-1} \append replicate(m(k)-k+1, (o, r_{\min})), & \mbox{if $i\ge k$ and $a_k \ne b$;}
      \end{array}\right.
 \end{align*}
 where
\[
replicate(n, x) = \underbrace{x : ... : x : []}_n
\]
and $o\in{\cal O}$ can be chosen arbitrarily.
 $q(b)$ returns a fixed sequence $e_{1..k-1}$ without seeing the action sequence until time $k-1$, but the time $k$ is the judgment day.
 $a_k=b$ promises the eternal heaven, while other selections result in the eternal hell.
 Let us consider the behavior of the term on $q(b)$ in Eq.~\eqref{eq:kore}.

$q(b)$ satisfies $\dbl{q(b)}(a_{1..k-1}) = e_{1..k-1}$. %
The term on $q(b)$
\[
 \left(\sum_{t=k}^{m(k)}r(e_t)\right) \xi_2(q(b)) = \left\{
 \begin{array}{l l}
  (m(k)-k+1)r_{\max}\xi_2(q(b)), & \mbox{ if $a=b$}; \\
  (m(k)-k+1)r_{\min}\xi_2(q(b)), & \mbox{ if $a\ne b$}.
 \end{array}
 \right.
\]
remains in the final $Q(\masae_{1..k-1}, a)$, no matter which actions are selected as $a_{k+1}$ ... $a_m(k)$ by the $\max$ operations.
Since from the premise $r_{\min}\ne r_{\max}$, the coefficient for $\xi_2(q(b))$ takes different positive rational values depending on whether $a=b$ or not.

 From the construction of $\xi_2$, the fraction $d(\eta(q(b)))$ of $\xi_2(q(b))$ is a random real value independent of the fraction $d(\eta(q_1))$ of $\xi_2(q_1)$ for any $q_1$ such that $q_1\ne q(b)$.
Thus, $Q(\masae_{1..k-1}, a) \ne Q(\masae_{1..k-1}, b)$ with probability 1 for any $a\ne b$.
 
 Moreover, although the domain $({\cal A}\times{\cal E}){*} \times {\cal A}$ of $Q$ is infinitely countable when $({\cal A}\times{\cal E}){*}$ is limited to finite lists,
$a \mapsto Q(\masae_{1..k-1}, a)$ is monomorphic for any $\masae_{1..k-1}$ with probability 1, because the $Q$ value can be taken randomly from a continuum.
 \end{proof}

\end{document}